\documentclass{article}
\usepackage[bindingoffset=0cm,a4paper,centering,textheight=200mm,textwidth=120mm]{geometry}
\usepackage{amstext,amssymb,latexsym}
\usepackage{graphicx,tabularx}
\usepackage{amsmath}
\usepackage{showlabels}
\usepackage{datetime}

\usepackage[colorlinks=true,%
            linkcolor=blue,%
            citecolor=blue,%
            urlcolor=blue]{hyperref}

\newtheorem{definition}{Definition}[section]
\newtheorem{theorem}[definition]{Theorem}
\newtheorem{remark}[definition]{Remark}

\newtheorem{lemma}[definition]{Lemma}
\newtheorem{construction}[definition]{Construction}
\newtheorem{example}[definition]{Example}
 \newenvironment{proof}{\begin{trivlist}\item[]{\bf Proof.}}{\hspace*{\fill} $\blacksquare$ \end{trivlist}}

 \newcommand{\eqinv}{\mbox{Eq$_{\rm inv}$}}
 
\newcommand{\w}[1]{\text{#1}}
\newcommand{\ws}[1]{~\text{#1}~}

\ifx\text\undefined
  \newcommand{\text}[1]{\relax
    \ifmmode\mathchoice
      {\hbox{\the\textfont0\relax#1}}%
      {\hbox{\the\textfont0\relax#1}}%
      {\hbox{\the\scriptfont0\relax#1}}%
      {\hbox{\the\scriptscriptfont0\relax#1}}%
    \else{\relax#1}\fi}
  \fi
\newcommand{\comma}{,\ldots, }
\newcommand{\half}{\frac{1}{2}}

\newcommand{\Bf}{{\bf f}}
\newcommand{\Bm}{{\bf m}}
\newcommand{\vare}{\varepsilon}
\newcommand{\CX}{\mathcal{X}}
\newcommand{\CA}{\mathcal{A}}
\newcommand{\CP}{\mathcal{P}}
\newcommand{\CB}{\mathcal{B}}
\newcommand{\CE}{\mathcal{E}}

\newcommand{\CF}{\mathcal{F}}

\newcommand{\CJ}{\mathcal{J}}

\newcommand{\inc}{\w{{\bf in}}}
\newcommand{\exc}{\w{{\bf out}}}
\newcommand{\und}{\w{{\bf und}}}
\newcommand{\tuple}[1]{\w{$\langle #1 \rangle$}}
\renewcommand{\iff}{\leftrightarrow}

\newcommand{\eqr}[1]{\w{\upshape (#1)}}

\begin{document}
\title{Probabilistic Argumentation. An Equational Approach}

\date{}

\author{D. Gabbay\\
King's College London,\\ Department of Informatics,\\ 
London, WC2R 2LS, UK, and\\
Department of Computer Science,\\
Bar Ilan University, and\\
University of Luxembourg\\
{\tt dov.gabbay@kcl.ac.uk}
\and
O. Rodrigues\\
King's College London,\\ Department of Informatics,\\ 
London, WC2R 2LS, UK\\
{\tt odinaldo.rodrigues@kcl.ac.uk}\\
}
\maketitle    

\begin{abstract}
There is a generic way to add any new feature to a system. It involves
1) identifying the basic units which build up the system and 2)
introducing the new feature to each of these basic units. 

In the case where the system is {\em argumentation} and the feature 
is probabilistic we have the following. The basic units are: {\bf a.} 
the nature of the 
arguments involved; {\bf b.} the membership relation in the set $S$ of
arguments; {\bf c.} the attack relation; and {\bf d.} the choice of
extensions.

Generically to add a new aspect (probabilistic, or fuzzy, or
temporal, etc) to an argumentation network \tuple{S,R} can be done by
adding this feature to each component {\bf a}--{\bf d}. This is a
brute-force method and may yield a non-intuitive or meaningful result.

A better way is to meaningfully translate the object system into
another target system which does have the aspect required and then let
the target system endow the aspect on the initial system. In our case
we translate argumentation into classical propositional logic and get
probabilistic argumentation from the translation.

Of course what we get depends on how we translate.

In fact, in this paper we introduce probabilistic semantics to
abstract argumentation theory based on the equational approach to
argumentation networks. We then compare our semantics with existing
proposals in the literature including the approaches by M. Thimm and
by A. Hunter. Our methodology in general is discussed in the
conclusion.
\end{abstract}

\section{Introduction}
The objective of this paper is to provide some orientation
to underpin probabilistic semantics for abstract argumentation. We feel that
a properly developed probabilistic argumentation framework cannot be obtained
by simply imposing an arbitrary probability distribution on the components 
of an argumentation system that does not agree with the dynamic aspects
of these networks. We need to find a probability distribution that is 
compatible with their underlying motivation.

We shall use the methodology of ``Logic by Translation'', which works
as follows: Given a new area for which we want to study certain aspect
properties AP, we translate this area to classical logic, study AP in
classical logic and then translate back and evaluate what we have
obtained. 


Let us start by looking at interpretations of an abstract
argumentation network $\tuple{S,R}$, $S \neq \varnothing$, $R\subseteq
S\times S$, into logics which already have probabilistic
versions. This way we can import the probability aspect from there and
it will have a meaning.  We begin with translating abstract
argumentation frames into classical propositional logic. In the
abstract form, the elements of $S$ are just atoms waiting to be
instantiated as arguments coming from another application system. $R$
may be defined using the source application system or may represent
additional constraints. At any rate, in this abstract form, $S$ is just
a set of atoms and all we have about it is $R$. In translating
\tuple{S,R} into classical propositional logic, we view $S$ as a set
of atomic propositions and we use $R$ to generate a classical theory
$\Delta_{\tuple{S,R}}$. Consider Figure~\ref{532-F1},
which describes the basic attack formation of all the attackers
$Att(x) =\{ y \in S \; | \; (y,x) \in R\}= \{y_1\comma y_n\}$ of 
the node $x$ in a network \tuple{S,R}. 

\begin{figure}
\centering

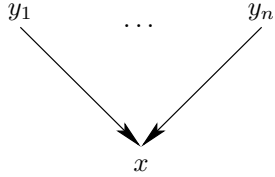

\caption{Basic attack formation in an argumentation network.}\label{532-F1}
\end{figure}

The essential logic translation of the attack on each node $x$ is given by 
\eqr{E1} below, where $x,y_i$ are propositional symbols representing the
elements $x,y_i \in S$:

\[
x \leftrightarrow \bigwedge_i\neg y_i\eqno{\eqr{E1}}
\]

So \tuple{S,R} corresponds to a classical {\em propositional} 
theory $\Delta_{\tuple{S,R}}=\{ x \leftrightarrow \bigwedge_i\neg y_i \; 
| \; x \in S\}$.\footnote{If there is a logical relationship between the 
arguments of $S$ that can be captured by formulae, then we can alternatively 
instantiate $x \longmapsto \varphi_x$, giving $\Delta_{\tuple{S,R}}=\{\varphi_x 
\iff \bigwedge_i \neg \varphi_y \; | \; x,y \in S \}$.}
Note that in classical logic, this theory may be 
inconsistent and have no models. For example, if $S$ contains a single 
node $x$ and $R$ is $\{(x,x)\}$, i.e., the network has a single
self-attacking node, then the associated theory is  $\{x \iff \neg x\}$, 
which has no model. For this reason it is convenient to regard these 
theories as theories of Kleene three-valued logic, with values in 
$\{ 0, \frac{1}{2}, 1\}$. In this 3-valued semantics, a valuation
would satisfy $x \iff \neg x$ if and only if it gives the value 
$\frac{1}{2}$ to $x$.\footnote{In Kleene's logic, one can interpret 
$\neg$ as complement to $1$; $\wedge$ as $\min$; and $\vee$ as $\max$. 
Thus, if the values of $A,B$ are $v(A),v(B)$, then $v(\neg A) = 1- v(A)$, 
$v(A \wedge B)= \min(v(A),v(B))$ and $v(A \vee B)=\max(v(A),v(B))$.}

If we consider the {\em equational approach} \cite{5}, then we can write
\[x =\bigwedge_i \neg y_i \eqno {\eqr{E2}}\] 
where \eqr{E2} is a numerical equation over the real interval $[0,1]$, with
conjunction and negation interpreted as numerical functions expressing the 
correspondence of the values of the two sides. 

A complete extension of $\tuple{S,R}$ is a solution to the equations of the
form of \eqr{E2} when they are viewed as a set of Boolean equations
in Kleene's 3-valued logic with values $\left\{0, \half, 1\right\}$, where
\begin{eqnarray}
x=0    &\ws{means that}& x=\exc \ws{(at least one attacker $y_i = \inc$)} \label{eq:c1}\\[1ex] 
x=1    &\ws{means that}& x=\inc \ws{(all attackers $y_i = \exc$)} \label{eq:c2}\\[1ex]
x=\half&\ws{means that}& \w{\parbox{1.5cm}{$x=\und$}}\ws{(no attacker $y_i=\inc$ and at least} \label{eq:c3}\\[-1.5ex]
       &               & \w{\parbox{1.5cm}{~}}      \ws{~one attacker $y_j = \und$)}\nonumber
\end{eqnarray}

The acceptability semantics above can be re-written in terms of the 
semantics of Kleene's logic as 
\[
v(x)=\min\{1 - v(y_i)\}
\]
which in equational form can be simplified to
\[
x=1-\max \{y_i\} \eqno\eqr{E2*}
\]
The reader should note that we actually solve the equations over  
the unit interval $[0,1]$ and project onto Kleene's 3-valued logic 
by letting 
\[\begin{array}{llll}
x =0     &\mbox{mean} &x=\exc &\w{(at least one attacker $y_i = \inc$)}\\[1ex] 
0 < x < 1&\mbox{mean} &x=\und &\w{(no attacker $y_i=\inc$ and at least}\\[0.2ex]
         &            &       &\w{~one attacker $y_j = \und$)}         \\[1ex]
x = 1    &\mbox{mean} &x=\inc &\w{(all attackers $y_i = \exc$)}        \\[1ex] 
\end{array}
\]
Now there are probabilistic approaches to two-valued classical
logic. The simplest two methods are described in Gabbay's book
{\em Logic for Artificial Intelligence and Information Technology} 
 \cite{1}.
Our idea is to bring the probabilistic approach through the above 
translation into argumentation theory.

Let us start with a description of the probabilistic approaches to
classical propositional logic.
\paragraph{Method 1: Syntactic.}
Impose probability $P (q)$ on the atoms $q$ of the language and propagate
this probability to arbitrary well-formed formulas (wffs).  So if
$\varphi (q_1\comma q_m)$ is built up from the atoms $q_1\comma q_m$,
we can calculate $P(\varphi)$ if we know $P(q_i)$, $i=1\comma m$.

\paragraph
{Method 2: Semantic.}  Impose probability on the models of the
language of $\{q_1\comma q_m\}$.  The totality of models is the space
$W$ of all $\{0,1\}$-vectors in $2^m$.  We give values $P(\vare)$, for
any $\vare\in 2^m$, with the restriction that $\Sigma_{\vare\in 2^m}
P(\vare)=1$.  The probability of any wff $\varphi$ is then
\[
P(\varphi)=\Sigma_{\vare\Vdash \varphi} P(\vare) \eqno \eqr{P1}
\]

The motivation for the syntactical Method 1 is that the atoms
$\{q_1\comma q_m\}$ are all independent. So for example, the date of
birth of a person ($p$) is independent of whether it is going to rain
heavily on that person's 21st birthday ($q$). However, if we want to
hold a birthday party $r$ in the garden on the 21st birthday, then we
have that $q$ attacks $r$.

If, on the other hand, we have:
\begin{quote}
$a=$ John comes to the party\\[1ex]
 $b=$ Mary comes to the party
\end{quote}
then $a$ and $b$ may be dependent, especially if some relationship
exists between John and Mary. We may decide that the probability of
$a\wedge b$ is $0$, but the probabilities of $\neg a\wedge b$ and of
$a\wedge\neg b$ are $\frac{1}{4}$ each and the probability of $\neg
a\wedge\neg b$ is $\half$. Assigning probability in this way depends
on the likelihood we attach to a particular situation (model). This is the 
semantic approach.

Example~\ref{ex:p1} shows that these two methods are orthogonal.

\begin{example}
\label{ex:p1}
What can $\Delta_{\tuple{S,R}}$ mean in classical logic? It is a 
generalisation of the ``Liar's paradox''. $x$ attacking itself is
like $x$ saying ``I am lying'': $x=\top$ if and only if $x=\bot$.
Figure~\ref{532-F1} represents $y_i$ saying $x$ is a lie.
$\Delta_{\tuple{S,R}}$ represents a system of lying accusations:
a {\em community liar paradox}.

Similarly, $S$ can represent people possibly invited to a birthday
party. $y \rightarrow x$ means $y$ saying ``if I come, $x$ cannot
come''. So Figure~\ref{532-F1} is saying ``invite $x$ if and only if 
you do not invite any of the $y_i$''. 

Suppose we instantiate $x \longmapsto \varphi_x$. Then we must
have $P(\varphi_x)=P(\bigwedge_i \neg \varphi_{y_i})$. However, there
may be also a connection between $\varphi_x$ and some $\varphi_{y_k}$, 
e.g., $\varphi_x \vdash \varphi_{y_k}$. This will impose further
restrictions on $P(\varphi_x)$ and $P(\varphi_{y_k})$, and it may be
the case that no such probability function exists.
\end{example}

\begin{remark}\label{532-R2}
The two approaches are of course, connected. If we are given a
probability on each $q_i$, then we get probability on each $\vare\in
2^m$ by letting
\[
P(\vare) =\Pi_{\vare\Vdash q} P (q) \times \Pi_{\vare\Vdash \neg q}
(1-P(q)) \eqno \eqr{P2}
\]
The $q_i$'s are considered independent, so the probability of
$\bigwedge_i\pm q_i$ is the product of the probabilities
\[
\textstyle P(\bigwedge_i \pm q_i)=\Pi_iP(\pm q_i)
\]
where $P (\neg q_i) =1-P(q_i)$ and the probability of $A\vee B$ is
\[
P(A\vee B) =P(A) +P(B)
\]
when $\Vdash \neg (A\wedge B)$, as is the case with disjuncts in a
disjunctive normal form.

So, for example
\[\begin{array}{rcl}
P((a\wedge b)\vee (a\wedge\neg b)) & = &P(a\wedge b) +P(a\wedge\neg b)\\[1ex]
 & = & P (a)P(b) +P(a) (1-P(b))\\[1ex]
 & = & P(a)(P(b) +1-P(b))\\[1ex]
 & = & P (a).
\end{array}
\]
\end{remark}

\section{The syntactical approach (Method 1)}

Let us investigate the use of the syntactical approach.

Let \tuple{S,R} be an argumentation network.  In the equational approach, 
according to the syntactical Method 1, we assign probabilities to all the atoms and are
required to solve the equation \eqr{E3} below for each $x$,  where
$Att(x)=\{y_i\}$ and $x$ and all $y_i$ are numbers in $[0,1]$:
\[
P(x) =P(\bigwedge_i\neg y_i), \eqno{\eqr{E3}}
\]
Since in Method 1, all atoms are independent, \eqr{E3} is equivalent to \eqr{E3*}:
\[
P(x) =\Pi_i(1-P (y_i)).  \eqno\eqr{E3*}
\]

Such equations always have a solution.

Let us check whether this makes sense.  Let us try to identify the
argument $x$ equationally with its probability, namely we let $P(x)=x$.
\begin{quotation}
\noindent If $x=\inc$, let $P(x) =1$\\[1ex]
If $x=\exc$, let $P(x)=0$.\\[1ex]
If $x=\und$, let $0< P(x) < 1$ 
\end{quotation}
to be determined by the solution to the equations.

Equation \eqr{E3*} becomes, under $P(x) =x$, the following:
\[
x=\Pi (1-y_i) \ws{for} x\in S.  \eqno{\eqr{E4}}
\]

This is the \eqinv\ equation in the equational approach (see
\cite{5}).

The following definition will be useful in the interpretation
of values from $[0,1]$ and their counterparts in Caminada's labelling 
functions.


\begin{definition}
\label{def:CP-GR-translation}
A valuation function $\Bf$ can be mapped into a labelling function $\lambda(\Bf)$ as
follows.

\begin{center}
\begin{tabular}{ccc}
\hline
$\Bf(x)=1$ & $\rightarrow$ & $\lambda(\Bf)(x)=\inc$ \\[1ex]
$\Bf(x)=0$ & $\rightarrow$ & $\lambda(\Bf)(x)=\exc$ \\[1ex]
$\Bf(x)\in (0,1)$ & $\rightarrow$ & $\lambda(\Bf)(x)=\und$ \\[1ex] \hline
\end{tabular}
\end{center}
\end{definition}

What do we know about \eqinv?  We quote the following from \cite{5}.

\begin{theorem}\label{532-f1}
Let \Bf\ be a solution to equations \eqr{E4}. Then $\lambda(\Bf)$ defined according
to Definition~\ref{def:CP-GR-translation} is a legal Caminada labelling (see \cite{4}) 
and leads to a complete extension.
\end{theorem}

\begin{theorem}\label{532-f2}
Let $\lambda_0$ be a legal Caminada labelling leading to a preferred extension.  
Then there exists a solution $f_0$, such that $\lambda_0 =\lambda (f_0)$.
\end{theorem}

\begin{figure}[htb]
\centering

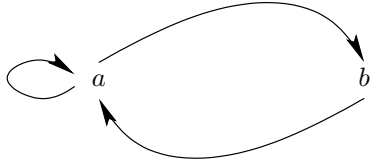

\caption{A sample argumentation network having a complete extension that cannot be
found via Equations~\eqr{E4}.\label{532-F3}}
\end{figure}

\begin{remark}\label{532-f3}
There are (complete) extensions $\lambda'$ such that there does not exist an $f'$
with $\lambda'=\lambda(f')$.

For example, in Figure \ref{532-F3}, the extension $a=b=\und$ cannot
be obtained by any $f$. Only $b=\inc$, $a=\exc$ can be obtained as a
solution to equations \eqr{E4}.\footnote{The equations are
\begin{enumerate}
\item $a=(1-a) \times (1-b)$
\item $b=1-a$.
\end{enumerate}
From the above two equations we get
\begin{enumerate}
\setcounter{enumi}{2}
\item $a=(1-a)\times a$
\end{enumerate}
The only possibility is $a=0$.}
\end{remark}

\begin{example}\label{532-f4}
Let \tuple{S,R} be given and let $\lambda$ be a complete extension which is
not preferred! The reason that $\lambda$ is not preferred, is that we have by 
definition, a $\lambda_1$ extending $\lambda$, which gives more $\{\inc,\exc\}$ 
values to points $z$, for which $\lambda$ gives the value \und. Therefore, we can 
prevent the existence of such an extension $\lambda_1$, if we force such points $z$ to 
be undecided. This we do by attacking such points $z$ by a new self-attacking point $u$.
The construction is therefore as follows. We are given \tuple{S,R} and a complete 
extension $\lambda$, which is not preferred. We now construct a new \tuple{S',R'} 
which is dependent on $\lambda$. Consider \tuple{S',R'} where $S' = S\cup \{u\}$, 
where $u \not\in S$, is a new point.  Let $R'$ be
\[
R'=R\cup \{(u,u)\}\cup \{(u,v)\; | \; v \in S \ws{and} \lambda (v) =\und\}.
\]

Then $\lambda' =\lambda\, \cup\, \{(u, \und)\}$ is a preferred extension
of \tuple{S', R'} and can therefore be obtained from a function $f'$ using
the equations~\eqr{E4}.

Let us see what the construction above does to our example in Figure
\ref{532-F3}, and let us look at the extension $\lambda (a) =\lambda
(b) =\und$.

\begin{figure}
\centering

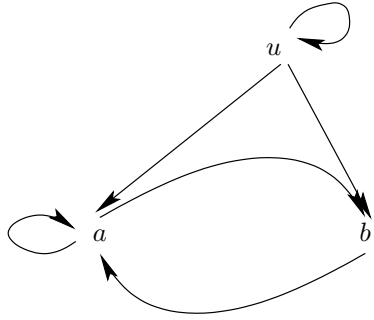
\caption{The network of Figure~\ref{532-F3} with an extra undecided node $u$
attacking all nodes.\label{532-F4}}
\end{figure}

Consider the network in Figure \ref{532-F4}. Its equations~\eqr{E4} are:
\begin{enumerate}
\item $u =1-u$
\item $a =(1-u) (1-a) (1-b)$
\item $b = (1-u)(1-a)$
\end{enumerate}
From (1) we get $u =\half$. So we have:

\begin{enumerate}
\setcounter{enumi}{1}
\item $a=\half (1-a) (1-b)$
\item $b=\half (1-a)$
\end{enumerate}

\[\begin{array}{lcl}
\renewcommand{\arraystretch}{1} 1-b &=& 1-\half (1-a)\\[1ex]
&=& \frac{2-1+a}{2} \\[1ex]
&= & \frac{1+a}{2}
\end{array}
\]
therefore substituting in (1) we get

\[\begin{array}{rcl}
a &=& \half (1-a) (\frac{1+a}{2})\\[1.2ex]
  &=& \frac{1}{4} (1-a^2)\\[1.2ex]
  & &4a+a^2-1=0\\[1.2ex]
  & & (a+2)^2-4-1=0\\[1.2ex]
  & & (a+2)^2=5\\[1.2ex]
 a&=& \sqrt{5}-2\approx0.236\\[1.2ex]
 b&=& \half (1-a)\\[1.2ex]
  &=&\half (1-\sqrt{5} +2)\\[1.2ex]
  &=&\frac{3-\sqrt{5}}{2} \approx 0.382.
\end{array}
\]
The extension of the network is $a=b=\und$.
\end{example}

\paragraph{Summary of the results  so far for the syntactical probabilistic 
method.}

Given an argumentation network \tuple{S,R}, we can find all Method 1
complete probabilistic extensions for it by solving all \eqinv\ 
equations.  Such complete probabilistic extensions will also be complete 
extensions in the traditional sense (i.e., Dung's), which will also include 
all preferred extensions (Theorems \ref{532-f1} and \ref{532-f2}).%
\footnote{Note that in traditional Dung semantics a preferred
extension $E$ is maximal in the sense that there is no extension $E'$ 
such that
\begin{enumerate}
\item If $x$ is considered \inc\ (resp. \exc) by $E$ then $x$ is also 
considered \inc\ (resp. \exc) by $E'$.
\item There exists at least one node considered \inc\ (resp. \exc) by $E'$ 
and considered \und\ by $E$.
\end{enumerate}
The above definition holds for numerical or probabilistic semantics,
where the value $1$ (resp. $0$) is understood as \inc\ (resp. \exc) and
values in $(0,1)$ are understood as \und.}

However, not all complete extensions can be obtained in this
manner (i.e., by Method 1, see remark \ref{532-f3} and compare with
Example \ref{532-E12}).

We can, nevertheless, for any complete extension $E$ which cannot be
obtained by Method 1, obtain it from the solutions of the equations generated
for a larger network \tuple{S',R'} as shown in Example \ref{532-f4}.

We shall say more about this in a later section.

\begin{remark}
Evaluation of the results so far for the syntactical probabilistic method.

\begin{enumerate}
\item \label{it:rem-1} We discovered a formal mathematical connection between 
  the syntactical probabilistic approach (Method 1) and the Equational
  \eqinv\ approach. Is this just a formal similarity or is there also
  a conceptual connection?

The traditional view of an abstract argumentation frame \tuple{S,R}, is
that the arguments are abstract, some of them abstractly attack each
other. We do not know the reason, but we seek complete extensions of
arguments that can co-exist (i.e., being attack-free), and that
protect themselves. The equational approach is an equational way of
finding such extensions. Each solution \Bf\ to the equations give rise to
a complete extension. The numbers we get from such solutions \Bf\ of
the equational approach can be interpreted as giving the degree of being
in the complete extension (associated with \Bf) or being out of it. 

Due to the mathematical similarity with the probability approach, these 
numbers are now interpreted as probabilities.  

To what extent is this justified? Can we do this at all?  

Let us recall the syntactical probabilistic method.
We start with an abstract argumentation framework \tuple{S,R} and 
add the probability $P(x)$ for each $x \in S$. We can interpret $P(x)$ as 
the probability that $x$ ``is a player'' to be considered (this is a vague 
statement which could mean anything but is sufficient for our purpose). The 
problem is how do we take into account the attack relation? Our choice was to 
require equation \eqr{E3}.  It is this choice that allowed the connection between
the syntactical probabilistic approach and the Equational
approach with \eqinv.  

So our syntactical probabilistic approach should work as
follows.  

Let $P$ be the independent probability on each $x\in S$. This is
an arbitrary number in $[0,1]$. Such a $P$ cannot be used for calculating
extensions because it does not take into consideration the attack
relation $R$. So modify $P$ to a $P'$ which does respect $R$ via  
Equation \eqr{E3}.  

How do we modify $P$ to find $P'$?  

Well, we can use a numerical iteration method. The details are not important 
here, the importance is in the idea, which can be applied to the traditional 
notion of extensions as well. Given \tuple{S,R} and an arbitrary desired assignment 
$E$ of elements that are \inc\ (and consequently also determining elements that are 
\exc) for $S$, this $E$ may not be legitimate in taking into account $R$, so we 
need to modify it to get the best proper extension $E'$ nearest to $E$ (cf.
\cite{Caminada-Pigozzi:11,gabbay-rodrigues:COMMA2014}).

So our syntactical probabilistic approach yielding a $P$ satisfying
Equation \eqr{E3} can be interpreted as \eqinv\ extensions obtained from
initial values which are probabilities (as opposed to, say, initial
values being a result of voting) corrected via iteration procedures
using $R$.  

Alternatively, we can look at the \eqinv\ equations as a
mathematical means of finding all those syntactical probabilities $P$
which respect the attack relation $R$ (via Equation~\eqr{E3}). 

Or we can see the solutions of the \eqinv\ as giving probabilities for
being included or excluded in the complete extension defined by these solutions 
(as opposed to the interpretation of the degree of being \inc\ or \exc).

\item The discussion in item~\ref{it:rem-1}. above hinged upon the choice 
  we made to take account of $R$ by respecting Equation~\eqr{E3}. There are other
  alternatives for taking $R$ into account. We can give direct, well-motivated 
  definitions of how to propagate probabilities along attack
  arrows. This is similar to the well-known problem of how to
  propagate probabilities along proofs (provability support arrows,
  or modus ponens, etc). Such an analysis is required anyway for
  instantiated networks, for example in ASPIC+ style \cite{Modgil-Prakken:14}). 
  We shall deal with this in a subsequent paper.
\end{enumerate}
\end{remark}

\section{The semantical approach (Method 2)}

Let us now check what can be obtained if we use Method 2, i.e., giving
probability to the models of the language. In this case the equation
(for $\{y_i\} = Att(x)$) \eqr{E3} $P(x) = P(\bigwedge_i\neg y_i)$ still
holds, but the $\neg y_i$ are not independent. So we cannot write
equation \eqr{E3*} for them and get \eqinv. Instead we need to use
the schema $P(A\vee B) = P(A) +P(B) -P(A\wedge B)$.  We begin with a
key lemma, which will enable us to compare later with the work of
M. Thimm, see \cite{thimm:12}.

\begin{lemma}\label{532-L5}
Let \tuple{S,R} be a network and let $P$ be a probability measure on
the space $W$ of all models of the language whose set of atoms is $S$. 
For $x\in S$, let the following hold
\[
P(x) = P(\bigwedge^n_{i=1} \neg y_i)
\]
where $Att(x)=\{y_1,\ldots,y_n\}$.

Then we have
\begin{enumerate}
\item $P(x) \leq P(\neg y_i), 1 \leq i \leq n$
\item $P(x) \geq 1 -\Sigma^n_{i=1} P(y_i)$
\end{enumerate}
\end{lemma}

\begin{proof}
By induction on $n$.
\begin{enumerate}
\item If $x =\neg y$ then $P(x) =1-P(y)$ and the above holds.
\item Assume the above holds for $m$, show for $m+1$. Let $z
  =\bigvee^m_{i=1} y_i, y=y_{m+1}$. Then $x=\neg z \wedge\neg y$.

We have by the induction hypothesis
\begin{itemize}
\item $P(\neg z) \leq P(\neg y_i), i=1\comma m$
\item $P(\neg z) \geq 1-\Sigma^m_{i=1} P(y_i)$
\end{itemize}
Consider now:
\[\begin{array}{lcl}
P(\neg z \wedge\neg y) &=& 1-P(y\vee z)\\[1ex]
 &=& 1-(P(y) +P(z) -P(y\wedge
z))\\[1ex]
 &=& 1-P(y)-P(z) +P(y\wedge z)\\[1ex]
 &=& 1-P(y)-(P(z)-P(y\wedge z))
\end{array}
\]
But $P(A\wedge B)\leq P(B)$ is always true.

So
\[
P(\neg z\wedge\neg y) \leq 1-P(y) =P(\neg y)
\]

On the other hand, by our assumption
\[
1-P(z) =P(\neg z) \geq 1-\Sigma^m_{i=1} P(y_i)
\]
So
\[\begin{array}{lcl}
P(\neg z \wedge\neg y) &=& 1-P(y) -P(z) +P(y\wedge z)\\[1ex]
 && (1-P(z))
-P(y) +P(y\wedge z)\\[1ex]
 &\geq & 1-\Sigma P(y_i) -P(y) +P(y\wedge
z)\\[1ex]
 &\geq & 1-\Sigma^{m+1}_{i=1} P(y_i)
\end{array}
\]
\end{enumerate}
\end{proof}

\begin{remark}\label{532-R5}
The converse of Lemma \ref{532-L5} does not hold, as we shall see in
Example \ref{532-E10} below.
\end{remark}

Let us look at some examples illustrating the use of Method 2.

\begin{example}\label{532-E6}
Consider the network in Figure \ref{532-F7}.  This figure is taken
from Thimm's ``A probabilistic semantics for abstract argumentation'' 
\cite[Figure 1]{thimm:12}.  We include it here for two reasons:
\begin{enumerate}
\item To illustrate or probabilistic semantic approach.
\item To use it later to compare our work with Thimm's approach.
\end{enumerate}

\begin{figure}
\centering 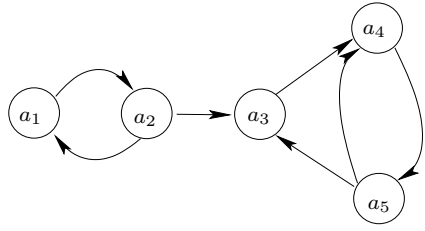
\caption{Figure 1 of ``A probabilistic semantics for abstract argumentation''
\cite{thimm:12}.}\label{532-F7}
\end{figure}

Let us apply Method 2 to it and assign probabilities to the models of
the propositional language with the atoms $\{a_1, a_2, a_3, a_4,
a_5\}$.  We assign $P$ as follows.
\[\begin{array}{l}
 P(a_1\wedge \neg a_2\wedge a_3\wedge\neg a_4\wedge a_5)=0.3\\[1ex]
 P(a_1\wedge\neg a_2\wedge\neg a_3\wedge a_4\wedge\neg a_5)=0.45\\[1ex]
 P(\neg a_1\wedge a_2\wedge\neg a_3\wedge \neg a_4\wedge a_5)=0.1\\[1ex]
 P(\neg a_1\wedge a_2\wedge\neg a_3\wedge a_4\wedge\neg a_5)=0.15\\[1ex]
 P(\w{any other conjunctive model})=0.
\end{array}
\]
  Let us compute $P(a_i)$, for $i=1\comma 5$.
  
  We have
  \[
  P(X) =\sum_{\vare \Vdash X} P (\vare).
  \]
  We get
  \[\begin{array}{l}
  P(a_1) = 0.3 + 0.45 = 0.75\\[1ex]
 P(a_2) = 0.1 + 0.15 = 0.25\\[1ex]
 P(a_3) =
  0.3\\[1ex]
 P(a_4) = 0.45 + 0.15 = 0.6\\[1ex]
 P(a_5) = 0.3 + 0.1 = 0.4.
  \end{array}
  \]
  
  To be a legitimate probabilistic model $P$ must satisfy equation
  \eqr{E3} relating to the attack relation of Figure \ref{532-F7}.  Namely
  we must have
  \[
  P(X) = P(\bigwedge_{Y\in Att(X)} \neg Y)  \eqno \eqr{E3}
  \]
  
  Therefore
\[\begin{array}{l}
P(a_1)=P(\neg a_2)\\[1ex]
P(a_2)=P(\neg a_1)\\[1ex]
P(a_3) = P(\neg a_2\wedge\neg a_5)\\[1ex]
P(a_4) = P(\neg a_3\wedge\neg a_5)\\[1ex]
P(a_5) = P(\neg a_4)
\end{array}
\]
  Let us calculate the $P$ in the right hand side of the above
  equations.
\[\begin{array}{l}
P(\neg a_2) = 1-0.25 = 0.75\\[1ex]
P(\neg a_1) = 1-0.75 = 0.25\\[1ex]
P(\neg a_2 \wedge\neg a_5) = 0.45\\[1ex]
P(\neg a_3 \wedge\neg a_5) = 0.45 + 0.15 = 0.6\\[1ex]
P(\neg a_4) = 0.4
\end{array}
\]
  
  We see that
  \[
  P(a_3) = 0.3 \neq P(\neg a_2 \wedge \neg a_5) = 0.45.
  \]
  
  Therefore this distribution $P$ is not legitimate according to our
  Method 2.  It does not satisfy equations \eqr{E3} because
  \[
  P(a_3)\neq P(\neg a_2\wedge \neg a_5)
  \]
  
  Therefore Lemma \ref{532-L5} does not apply and indeed, condition
  (2) of Lemma \ref{532-L5} does not hold for $a_3$.  We have $P(a_3)
  = 0.3$ but $1-P(a_2)-P(a_5)= 0.35$.
\end{example}

\begin{example}\label{532-E8}
Let us look at Figure \ref{532-F9}.  This is also taken from Thimm's
paper \cite[Figure 2]{thimm:12}.  It shall be used later to compare our
methods with Thimm's.

\begin{figure}
\centering 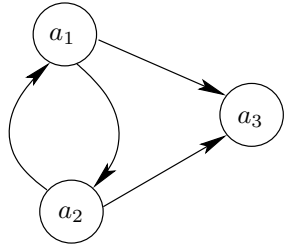
\caption{Figure 2 of ``A probabilistic semantics for abstract argumentation'' 
\cite{thimm:12}.}\label{532-F9}
\end{figure}

\paragraph{1.  We use Method 2.}
Consider the following probability distribution on models

\[\begin{array}{l}
P(a_1\wedge\neg a_2\wedge\neg a_3)=0.5\\[1ex]
 P(a_1\wedge\neg a_2\wedge
a_3)=0\\[1ex]
 P(a_1\wedge a_2\wedge\neg a_3)=0\\[1ex]
 P(a_1\wedge a_2\wedge a_3)
=0\\[1ex]
 P(\neg a_1\wedge a_2\wedge a_3)=0\\[1ex]
 P(\neg a_1\wedge
a_2\wedge\neg a_3)=0.5\\[1ex]
 P(\neg a_1\wedge\neg a_2\wedge
a_3)=0\\[1ex]
 P(\neg a_1\wedge \neg a_2\wedge\neg a_3)=0.
\end{array}
\]

In this model we get
\[\begin{array}{l}
P(a_1)=0.5\\[1ex]
 P(a_2) = 0.5\\[1ex]
 P(a_3) =0
\end{array}
\]
Let us check whether this probability distribution satisfies equation
\eqr{E3}, namely
\[
P(X) = P(\bigwedge_{Y\in Att(X)}\neg Y)  \eqno \eqr{E3}
\]

We need to have
\[\begin{array}{lcl}
P(a_1) &=& P(\neg a_2)\\[1ex]
 P(a_2) &=& P(\neg a_1)\\[1ex]
 P(a_3) &=& P(\neg
a_2\wedge \neg a_2)
\end{array}
\]
Indeed
\[\begin{array}{l}
P(\neg a_1) =1-P(a_1)=0.5\\[1ex]
 P(\neg a_2) =1-P(a_2) =0.5\\[1ex]
 P(\neg a_1\wedge \neg a_2)=0.
\end{array}
\]

Thus we have a legitimate model.

\paragraph{2.  We use Method 1.}
Let us use \eqinv\ on this figure, namely we try and solve the
equations
\[\begin{array}{l}
a_1=1-a_2\\[1ex]
 
a_2=1-a_1\\[1ex] 
a_3=(1-a_1)(1-a_2)
\end{array}
\]
Let us use a parameter $0 \leq x \leq 1$ and let
\[\begin{array}{l}
a_1=x,\\[1ex] a_2=1-x,\\[1ex] a_3=x(1-x)
\end{array}
\]
The probabilities we get with parameter $x$ as well as for $x=0.5$
are given below.

\[\begin{array}{lllll}
 P(a_1\wedge a_2\wedge a_3)& = & x^2(1-x)^2 & = & \frac{1}{16}\\[1ex]
 P(a_1\wedge a_2\wedge\neg a_3)& = & x(1-x)(1-x(1-x)) & = & \frac{3}{16} \\[1ex]
 P(a_1\wedge\neg a_2\wedge a_3) & = &  x^3(1-x) & = & \frac{1}{16}\\[1ex]
 P(a_1\wedge \neg a_2\wedge\neg a_3) & = & x^2(1-x(1-x)) & = & \frac{3}{16}\\[1ex]
 P(\neg a_1\wedge a_2\wedge a_3) & = & x(1-x)^3 & = & \frac{1}{16}\\[1ex]
 P(\neg a_1\wedge a_2\wedge\neg a_3) & = & (1-x)^2(1-x(1-x)) & = & \frac{3}{16}\\[1ex]
 P(\neg a_1\wedge\neg a_2\wedge a_3)& = & x^2(1-x)^2 & = & \frac{1}{16}\\[1ex]
 P(\neg a_1\wedge\neg a_2\wedge\neg a_3)& = & x(1-x)(1-x(1-x)) & = & \frac{3}{16}\\[1ex]
\end{array}
\]

If we choose $x=0.5$ we get $P(a_1)=P(a_2) =0.5$ and $P(a_3)=\frac{1}{4}$.
\end{example}

\begin{example}\label{532-E10}
This example shows that the converse of Lemma \ref{532-L5} does not
hold. Consider the network in Figure \ref{532-F11}.

\begin{figure}
\centering 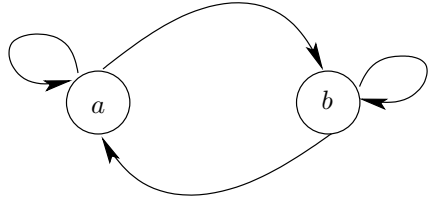
\caption{Network for Example~\ref{532-E10}.}\label{532-F11}
\end{figure}

Any legitimate probability assigned to models would be required to
satisfy the following
\[\begin{array}{l}
P(a)=P(\neg a\wedge\neg b)\\[1ex]
 P(b)=P(\neg a\wedge\neg b)
\end{array}\]

\paragraph{Case 1.}  Try the following probability $P_1$.

\[P_1(a\wedge b) = P_1(a\wedge \neg b) = P_1(\neg a\wedge b)=P_1(\neg a\wedge\neg b) = 0.25.
\]
Therefore
\[\begin{array}{l}
P_1(a) = 0.5\\[1ex]
P_1(b)=0.5
\end{array}
\]
Note that we also have
\[\begin{array}{l}
P_1(a) =\half \leq 1-P_1(b) =\half\\[1ex]
 P_1(a) =\half \leq 1-P_1(a)=\half.
\end{array}
\]
Similarly for $P_1(b)$ by symmetry.

Also
\[
P_1(a) =\half \geq 1-P_1(a)-P_1(b) =1-\half-\half=0.
\]
Thus the conditions of the conclusions of Lemma \ref{532-L5} hold.
However the assumptions of Lemma \ref{532-L5} do not hold, because
\[
P_1(a) =\half \neq P_1(\neg a\wedge\neg b)=\frac{1}{4}.
\]

\paragraph{Case 2.}  Let us check whether we can find a probability $P_2$ which is indeed acceptable to Method 2.  Let us try with variables $y,z$ and create equations and solve them:
\[\begin{array}{l}
P_2(a\wedge b)=y\\[1ex]
P_2(\neg a\wedge b)=z.
\end{array}
\]
Therefore $P_2(b) =y+z$.
\[
P_2(\neg a\wedge\neg b) = y+z
\]
and what is left is
\[
P_2(a\wedge\neg b) = 1-2y-2z
\]
but we must also have
\[P_2(a) =P_2(\neg a\wedge\neg b)
\]
and hence we must have
\[P_2(a) =1-2y-2z+y=P_2(\neg a\wedge\neg b)=y+z.
\]
So we get the equation
\[\begin{array}{l}
1-2y-3z=0\\[1ex]
 2y+3z=1\\[1ex]
 y=\frac{(1-3z)}{2}
\end{array}
\]
Since $0 \leq y, z \leq 1$ so $z$ must be less than $\frac{1}{3}$.

Let us choose $z=0.2$ and so $y=0.2$.

We get, for example
\[\begin{array}{l}
P_2(a\wedge b) = 0.2\\[1ex]
P_2(\neg a\wedge b) = 0.2\\[1ex]
P_2(\neg a\wedge\neg b) = 0.4\\[1ex]
P_2(a\wedge\neg b) = 0.2
\end{array}
\]
We could also have chosen $z=\frac{1}{3}$ and $y=0$. This would give
$P_3$, where
\[\begin{array}{l}
P_3(a\wedge b)=0\\[1ex]
P_3(\neg a\wedge b)=\frac{1}{3}\\[1ex]
P_3(\neg a\wedge\neg b) =\frac{1}{3}\\[1ex]
P_3(a\wedge\neg b)=\frac{1}{3}
\end{array}
\]
So we get
\[\begin{array}{l}
P_3(b)=P(a)=\frac{1}{3}\\[1ex]
P_3(\neg a\wedge\neg b)=\frac{1}{3}.
\end{array}\]
\end{example}

\begin{example}\label{532-E12}
Consider the network of Figure \ref{532-F3}.  Let us try to find a
probabilistic semantics for it according to Method 2.  Assume we have
\[\begin{array}{l}
P(a\wedge b) = x_1\\[1ex]
P(a\wedge\neg b) = x_2\\[1ex]
P(\neg a \wedge b) = x_3\\[1ex]
P(\neg a\wedge\neg b) = 1-x_1-x_2-x_3.
\end{array}
\]
We need to satisfy
\[\begin{array}{l}
P(a) = P(\neg a\wedge\neg b)\\[1ex]
 P(b) = P(\neg a)
\end{array}
\]
This means we need to solve the following equations.
\begin{enumerate}
\item $x_1+x_2 = 1-x_1-x_2-x_3$
\item $x_1+x_3=1-x_1-x_2$.
\end{enumerate}
By adding $x_1+x_2$ to both sides (1) can be written as
\[
2(x_1+x_2) =1-x_3,
\]
and by swapping $x_3$ to the right and $-x_1 -x_2$ to the left (2) can
be written as
\[
2x_1 +x_2 =1-x_3.
\]

Thus we get
\begin{enumerate}
\setcounter{enumi}{2}
\item $2x_1+x_2=2x_1+2x_2$.
\end{enumerate}
Therefore $x_2=0$.

There remains, therefore
\begin{enumerate}
\setcounter{enumi}{3}
\item $2x_1=1-x_3$.
\end{enumerate}
We can choose values for $x_3$.

\paragraph{Sample choice 1.}  $x_3=1$, so $x_1=0$.

We get $P_1(a\wedge b) = P(a\wedge\neg b) = P_1(\neg a\wedge\neg b) =0$ and
$P_1(\neg a\wedge b) =1$.

This yields $P(a) =0, P(b) =1$. This is also the \eqinv\ 
solution to
\[\begin{array}{l}
b = 1-a\\[1ex]
 a=(1-a)(1-b)
\end{array}
\]

\paragraph{Sample choice 2.}  $x_3=\half$.  So $x_1=\frac{1}{4}$ 
and the probabilities are
\[\begin{array}{l}
P_2(a\wedge b) =\frac{1}{4}\\[1ex]
P_2(a\wedge\neg b) =0\\[1ex]
P_2(\neg a\wedge b) =\half\\[1ex]
 P_2(\neg a\wedge\neg b) =\frac{1}{4}.
\end{array}
\]

$P_2$ is a Method 2 probability, which cannot be given by Method 1.

\paragraph{Sample choice 3.} $x_3=0$.  Then $x_1=\half$. We get 
\[\begin{array}{l}
P_3(a\wedge b)=\half\\[1ex]
 P_3(a\wedge\neg b) =0\\[1ex]
 P_3(\neg a\wedge b)=0\\[1ex]
 P_3(\neg a\wedge\neg b) =\half.
\end{array}
\]
Therefore $P_3(a)=P_3(b)=\half$.
\end{example}

\begin{lemma}\label{532-L13}
Let \tuple{S,R} be a network and let $P$ be a semantic probability
(Method 2) for \tuple{S,R}.  Let $x \in S$ and let $\{y_i\} = Att(x)$.
Then
\begin{enumerate}
\item If for some $y_i, P(y_i)=1$ then $P(x)=0$.
\item If for all $y_i, P(y_i)=0$ then $P(x)=1$.
\end{enumerate}
\end{lemma}

\begin{proof}
Let us use Figure \ref{532-F1} where $\{y_i\}=Att(x)$.

\paragraph{Case 1.} Assume that $P(y_1) =1$.  We need to show that $P(x) =0$.  We have:
\[
P(x) = P(\bigwedge_i \neg y_i) \eqno \eqr{E3}
\]
We also have
\[
P(A) =\sum_{\vare\Vdash A}P(\vare) \eqno \eqr{P1}
\]
Therefore
\[
P(x) = \sum_{\vare\Vdash \bigwedge_i\neg y_i}P(\vare)
\]

\[
P(x)=\sum_{\vare\Vdash\neg y_1\wedge\bigwedge^{n}_{j=1} \neg
  y_j}P(\vare) \eqno (i)
\]
but
\[
P(y_1)=\sum_{\vare\Vdash y_1}P(\vare)=1
\]
Therefore we have
\[
\sum_{\vare \Vdash \neg y_1} P(\vare) =0 \eqno (ii)
\]
From (i) and (ii) we get that $P(x) =0$.

\paragraph{Case 2.}   We assume that for all $i$, $P(y_i)=0$ and we need to show that $P(x) =1$.

We have
\[
P(x) = P(\bigwedge\neg y_i)
\]

\[P(x) =1-P(\bigvee y_i)
\eqno (iii)
\]
We also have
\[
P(\bigvee y_i)=\Sigma_{\vare\Vdash \bigvee y_i} P(\vare ) \eqno (iv)
\]

Suppose for some $\vare'$ such that $\vare'\Vdash \bigvee y_i$ we have
$P(\vare') > 0$.  But $\vare' \Vdash \bigvee y_i$ implies
$\vare'\Vdash y_i$, for some $i$.

Say $i=1$.  Thus we have $\vare' \Vdash y_1$ and $P(y_1) =0$ and
$P(\vare') > 0$. This is impossible since
\[
P(y_1)=\sum_{\vare\Vdash y_1}P(\vare) \eqno (v)
\]
Therefore for all $\vare$ such that $\vare\Vdash \bigvee y_i$ we have
that $P(\vare)=0$. Therefore by (iii) and (iv) we get
\[
P(x) =1.
\]
\end{proof}

\begin{remark}\label{532-R14}
Let \tuple{S,R} be a network and let $P$ be a semantic probability for
\tuple{S,R} (Method 2).

Let $\lambda$ be defined as follows, for $x\in S$.
\[
\lambda (x) = \left\{
\begin{array}{l}
\inc,\ws{if} P(x) =1\\[1ex]
\exc,\ws{if} P(x) =0\\[1ex]
\und,\ws{if} 0 < P(x) < 1
\end{array}
\right.
\]

The perceptive reader might expect us to say that $\lambda$ is a
legitimate Caminada labelling, especially in view of Lemma
\ref{532-L13}.  This is not the case as Example~\ref{532-E15}
shows.
\end{remark}

\begin{example}\label{532-E15}
This example shows that in the probabilistic semantics it is possible
to have $P(x)=0$, while for all attackers $y$ of $x$ we have $0 < P(y)
< 1$. Thus the nature of the probabilistic attack is different from
the traditional Dung one.  If $Att(x)$ is the set of all attackers of
$x$ and $P(\bigvee_{y\in Att(x)} y)=1$, then, and only then $P(x) =0$.

Thus the attackers of $x$ can attack with joint probability.

The example we give is the network of Figure \ref{532-F16}.

\begin{figure}
\centering 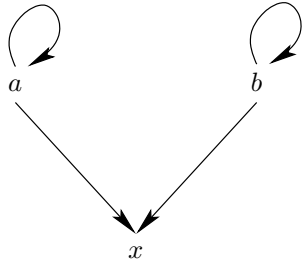
\caption{A network with Method 1 and Method 2 probabilities.}\label{532-F16}
\end{figure}

This has a Method 1 probability of $P_1(a) =\half, P_1(b)
=\frac{1}{4}$ and $P_1(x) =\frac{1}{4}$.

Thus for any model $\Bm =\pm a\wedge\pm b \wedge x$ we have
\[
P_1(\Bm) =\half \times \half \times \frac{1}{4}=\frac{1}{16}
\]
and for any model
\[
\Bm' =\pm a\wedge\pm \wedge\neg x
\]
we have
\[
P_1(\Bm)=\half \times\half\times \frac{3}{4}=\frac{3}{16}.
\]

Figure \ref{532-F16} also has a Method 2 probability model. We can
have
\[\begin{array}{l}
P_2(a) =P_2(b) =\half\\[1ex]
 P_2(x) =0.
\end{array}
\]
Let us check what values to give to the models. The models are:
\[\begin{array}{l}
m_1= x\wedge a\wedge b\\[1ex]
 m_2=x\wedge a\wedge\neg b\\[1ex]
 m_3 = x\wedge\neg
a\wedge b\\[1ex]
 m_4=x\wedge\neg a\wedge\neg b\\[1ex]
 m_5=\neg x\wedge a\wedge
b\\[1ex]
 m_6=\neg x \wedge a \wedge\neg b\\[1ex]
 m_7=\neg x \wedge\neg a\wedge
b\\[1ex]
 m_8=\neg x \wedge\neg a \wedge \neg b.
\end{array}
\]
We want the following equations to be satisfied.
\begin{enumerate}
\item $P_2 (x) =0$. This means we need to let
\[
P_2(m_i) =0, i=1,\ldots,4.
\]
\item $P_2(a) =\half$. This means we need to let
\[\begin{array}{l}
P_2(m_5) +P_2(m_6)=\half\\[1ex]
 P_2(m_7)+P_2(m_8)=\half.
\end{array}
\]
\item $P_2(b) =\half$, yields the equations
\[\begin{array}{l}
P_2(m_5) +P_2(m_7) =\half\\[1ex]
 P_2(m_6)+P_2(m_8) =\half.
\end{array}
\]
\item We also need to have the equation
\[
0=P_2(x) =P_2(\neg a \wedge\neg b)
\]
Therefore $P_2(m_8)=0$.

We thus have the following equations left
\begin{enumerate}
\item $P_2(m_5) +P_2(m_6)=\half$
\item $P_2(m_7) =\half$
\item $P_2(m_5)+P_2(m_7)=\half$
\item $P_2(m_6)=\half$.
\end{enumerate}
From (b) and ( c) we get $P_2(m_5) =0$.  This makes
$P_2(m_6)=\half$. Thus we get the following solution:
\[\begin{array}{l}
P_2(m_i) =0, \mbox{ for } i=1, 2, 3, 4, 5, 8\\[1ex]
 P_2(m_6) =P_2(m_7)
=\half.
\end{array}\]
\end{enumerate}

Note that the equations \eqr{E3} hold for $P_1$ and $P_2$:
\[\begin{array}{l}
P(a) =1-P(a)\\[1ex]
 P(b) =1-P(b)
\end{array}
\]
hold of both $P_1$ and $P_2$. As for $P(x) = P(\neg a\wedge\neg b)$ we
check
\[\begin{array}{lcl}
\frac{1}{4} =P_1(x) &=& P_1(\neg a \wedge\neg b)\\[1ex]
 &=& P_1(\neg
a)\times P_1 (\neg b) =\frac{1}{4}.
\end{array}
\]

For $P_2$ we have
\[\begin{array}{l}
0= P_2(x) =P_2(\neg a \wedge\neg b)\\[1ex]
 \qquad P_2(\neg (a\vee b)
=1-P_2(a\vee b)\\[1ex]
 P_2(a\vee b) = P_2(m_1)+P_2(m_2)\\[1ex]
 \qquad +P_2 (m_3)
+P_2(m_5) +P_2(m_6)\\[1ex]
 \qquad + P_2 (m_7) = 0+0+0+\half+\half=1.
\end{array}
\]
Thus $P_2(\neg a\wedge\neg b)=0$.

So $P_1$ and $P_2$ are legitimate probabilities on Figure
\ref{532-F16}.  $P_1$ is a Method 1 probability and $P_2$ is a
Method 2 probability.
\end{example}

\begin{definition}\label{532-D20}
We now define the Gabbay--Rodrigues Probabilistic Labelling $\Pi$ on a
network \tuple{S,R}.  $\Pi$ is a \{\inc, \exc, \und\}-labelling satisfying
the following.

There exists a semantic probability $P$ on \tuple{S,R} such that for
all $x\in S$
\begin{enumerate}
\item $\Pi(x)=\inc$, if $P(\bigvee Att(x)) =0$
\item $\Pi(x)=\exc$, if $P(\bigvee Att(x)) =1$
\item $\Pi(x)=\und$, if $0 < P (\bigvee Att(x)) < 1$
\end{enumerate}
\end{definition}

\begin{example}\label{532-E21}
This example is due to M. Thimm, oral communication, 24th October
2014.  Consider Figure~\ref{532-F22}.

\begin{figure}
\centering 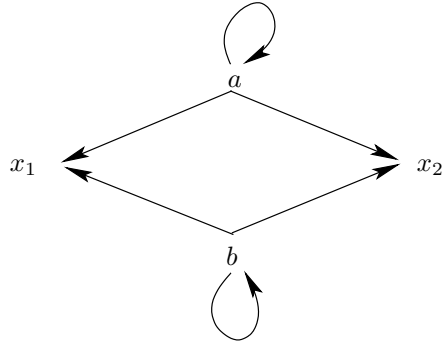
\caption{Mirrored network of Figure~\ref{532-F16}.}\label{532-F22}
\end{figure}

This figure contains Figure \ref{532-F16} and its mirror image. We saw
that in Figure \ref{532-F16} (as well as in this Figure \ref{532-F22})
any probability on the figures must yield
\[
P(a) =P(b) =\half.
\]
Figure \ref{532-F16} allowed for two possibilities for $x$. $P_1(x)
=\frac{1}{4}$ and $P_2(x) =0$. Let us try $P$ for our Figure
\ref{532-F22} with
\[ 
P(x_1)=\frac{1}{4}\mbox{ and } P(x_2) =0.
\]
This is not possible because we must have
\[
P(x_i) =P(\neg a\wedge\neg b).
\]
So $P(x_1)$ must be equal to $P(x_2)$.

This example will show in the comparison with the literature section
that our probability semantics is different from that of M. Thimm in
\cite{thimm:12}.

See also Example \ref{532-E10}.
\end{example}

\begin{theorem}\label{532-NT-1}
Let \tuple{S,R} be a network and let $\lambda$ be a legitimate
Caminada labelling on $S$, giving rise to a complete extension.  Then
there exists a probability $P_\lambda$ on the models (Method 2
probabilistic semantics) such that for all $x\in S$:
\begin{itemize}
\item $P_\lambda (x) =1$, if $\lambda (x)=\inc$
\item $P_\lambda (x)=0$, if $\lambda (x)=\exc$
\item $P_\lambda (x)=\half$, if $\lambda(x)=\und$.
\end{itemize}
\end{theorem}

\begin{proof}
(We use an idea from M. Thimm \cite{thimm:12})

Let $S=\{s_1\comma s_k\}$. Then when we regard the elements of $S$ as atomic
propositions in classical propositional logic, there are $2^k$ models based 
on $S$. Each of these models gives values $0$ (false) or $1$ (true) to each 
atomic proposition. Each such a model can be represented by a conjunction
of the form $\alpha=\bigwedge_i \pm s_i$. $\alpha$ represents the model which 
gives value $1$ to $s_i$ if $+s_i$ appears in $\alpha$  and gives value $0$ to 
$s_i$ if $-s_i$ appears in $\alpha$. Given 
a model we can construct the respective $\alpha$ for it. Let
\[
\displaystyle \alpha_1=\bigwedge_{\lambda(s)=\inc} s;\quad 
\displaystyle \alpha_0=\bigwedge_{\lambda(s)=\exc} \neg s;\quad 
\displaystyle \alpha_\half=\bigwedge_{\lambda(s)=\und} s;\quad \ws{and}
\displaystyle \beta_\half=\bigwedge_{\lambda(s)=\und} \neg s.
\]

We now define a Method 2 probability $P_\lambda$ on the models.
\begin{enumerate}
\item $P_\lambda (\alpha_1\wedge\alpha_0\wedge\alpha_\half) =\half$
\item $P_\lambda (\alpha_1\wedge\alpha_0\wedge\beta_\half) =\half$
\item $P_\lambda (m) =0$, for any other model, $m$ different from the
  above.
\end{enumerate}

Clearly $P_\lambda$ is a probability. We examine its properties
\begin{enumerate}
\renewcommand{\labelenumi}{(\roman{enumi})}
\item Let $x$ be such that $\lambda (x)=\inc$.

Then \[ P_\lambda (x) =\sum_{m\Vdash x} P_\lambda (m).
\]
Only (1) and (2) can contribute to $P_\lambda (x)$, so the value is 1.
\item Let $\lambda (x) =\exc$.

The only two models that can contribute to $P_\lambda (x)$ are in (1)
and (2) above, but they prove $\neg x$. So $P_\lambda (x) =0$.
\item Let $P_\lambda (x)=\und$.

Then clearly $P_\lambda (x)$ gets a contribution from (1) only. We get
$P_\lambda (x) =\half$.
\end{enumerate}

We now need to verify that $P_\lambda$ actually satisfies the
equations of \eqr{E3}.

Let $x\in S$ and let $y_i$ be its attackers. We want to show that
\[
P_\lambda (x) =P_\lambda (\bigwedge_i\neg y_i)
\]
or
\[
P_\lambda(x) =1-P_\lambda (\bigvee_i y_i).
\]
\begin{enumerate}
\renewcommand{\labelenumi}{(\roman{enumi})} \setcounter{enumi}{3}
\item Assume $P_\lambda (x) =1$. Then $P_\lambda (x)$ gets
  contributions from both (1) and (2). The only option is that then
  $\lambda (x)=$ in, and so all attackers of $y_i$ of $x$ are out, so
  $\alpha_0\Vdash \bigwedge\neg y_i$ and so $P_\lambda
  (\bigwedge_i\neg y_i)=1$, because it gets contributions from both
  (1) and (2).

\item Assume $P_\lambda (x) =0$.

Thus neither (1) nor (2) contribute to $P_\lambda (x)$. Therefore
$\alpha_0\Vdash x$ and so $\lambda (x) =$ out and so for some
attacker $y_i, \lambda (y_i)=$ in and so $\alpha_1\Vdash y_i$ and so
$P_\lambda (\bigwedge_i\neg y_i)$ cannot get any contribution either
from (1) or from (2) and so $P_\lambda (\bigwedge_i\neg y_i)=0$.

\item Assume that $P_\lambda (x) =\half$.

So $P_\lambda (x)$ can get a contribution either from (1) or from (2),
but not from both.  So $\lambda (x)$ must be undecided.

So the attackers $y_i$ of $x$ are either \exc\ (with $P_\lambda
(y_i)=0)$) or \und\ (with $P_\lambda (y_i) =\half$), and we have that
at least one attacker $y$ of $x$ is \und.
\end{enumerate}

Let $y^0_i$ be the attackers that are out and let $y^\half_j$ be the
undecided attackers. Consider
\[
e=\bigwedge_i \neg y^0_i\wedge\bigwedge_j \neg y^\half_j.
\]
The only model which can both contribute to $P_\lambda (e)$ is
$\alpha_1\wedge\alpha_0\wedge\beta_\half$ and thus $P_\lambda (e)
=\half$.

Thus from (iv), (v) and (vi) we get that \eqr{E3} holds for $P_\lambda$.
\end{proof}

\begin{remark}\label{532-NR-2}
Note that the $P_\lambda$ of Theorem \ref{532-NT-1} is strictly Method
2 probability.  For example we saw that the network of Figure
\ref{532-f2} with $a=b= $ und cannot solve Method 1 probability.  The
next section will see how far we can go with Method 1 probability.
\end{remark}

\paragraph{Summary of the results so far for the semantical probabilistic Method 2.}
We saw that Dung's traditional complete extensions strictly contain
the probabilistic Method 1 extensions and is strictly contained in the
probabilistic Method 2 extensions.

\section{Approximating the semantic probability by syntactic probability}
We have seen in Theorem \ref{532-NT-1} that the Method 2 probabilistic
semantics can give us all the traditional Dung complete
extensions. This result, together with the probabilistic semantics
$P_2$ of Example \ref{532-E15} would show that Method 2 semantics is
stronger than traditional Dung complete extensions semantics.

This section examines how far we can stretch the applicability of the
syntactical probability approach (Method 1). We know from the 
``all-undecided'' extension for the network in Figure~\ref{532-F3} that
there are cases where we cannot give Method 1 probability. We ask in
this section, can we approximate such extensions by Method 1
probabilities?

We find that the answer is yes.

Let \tuple{S,R} be a network. Let $\lambda$ be a legitimate Caminada
labelling giving rise to a complete extension $E=E_\lambda$. If the
extension is a preferred extension, then there exists a solution $f$
to the \eqinv\  equations which yield $\lambda$ and $f$ is
actually a Method 1 (and here also a Method 2) probabilistic semantics
for \tuple{S,R}.  The question remains as to what happens in the case
where $\lambda$ is not a preferred extension. In this case we are not
sure whether $\lambda$ can be realised by a solution $f$ of the\\[1ex]
 \quad 
\eqinv\ equations.  In fact there are examples of networks where no
such $f$ exists.  We know from Theorem \ref{532-NT-1} that there
exists a probability function $P_\lambda$ on models that would yield
$\lambda$ according to Definition \ref{532-D20}.  We seek an $\eqinv$
function which approximates this probability.

We shall use the ideas of Example \ref{532-f4}.

\begin{remark}\label{532-NR-3}
We need to use some special networks.

\begin{enumerate}
\item Consider Figure \ref{532-NF-4}, which we shall call $U_n$.  $n =
  1, 2, 3, \ldots$.

\begin{figure}
\centering 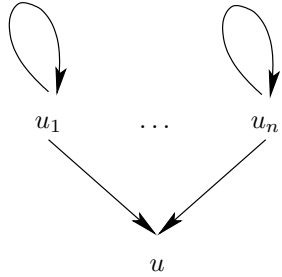
\caption{Multiple attacks by undecided nodes.}\label{532-NF-4}
\end{figure}

The $\eqinv$ equations solve for this figure as $u_i =\half, i=1\comma
n$.
\[u=\frac{1}{2^n}
\]
Thus if $u$ attacks any node $x$, its ``impact'' on $x$ is the
multiplicative value $1-\frac{1}{2^n}$. For $n$ very large, the attack
is almost negligible.
\item Let \tuple{S,R} be any network. Let $u$ be a node not in $S$. If
  we add $u$ to $S$ and let it attack all elements of $S$, we can
  assume in view of (1) above that the $\eqinv$ value of $u$ is
  $\frac{1}{2^n}$. Figure \ref{532-NF-5} depicts this scenario.

\begin{figure}
\centering 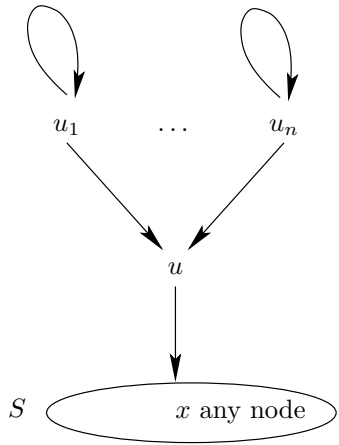
\caption{Scenario depicted in Remark~\ref{532-NR-3}.\label{532-NF-5}}
\end{figure}

We suppress $\{u_1\comma u_n\}$ and just record that $u
=\frac{1}{2^{n}}$.
\end{enumerate}

\end{remark}

\begin{construction}\label{532-C23}
Let \tuple{S,R} be given and let $\lambda$ be a legitimate Caminada
labelling giving rise to a non-preferred extension.

Let $u\not\in S$ be a new point and assume in view of Remark
\ref{532-NR-3} that the value of $u$ is very very small. Let
\[
S' = S\cup \{u\}
\] and let 
\[
R' = R\cup \{(u,v)|\lambda(v) = \und\}.
\]

Let $\lambda' = \lambda \cup \{(u,\und)\}$.

Let Att$(x)$ be the set of all attackers of $x$ in \tuple{S,R} and let
Att$'(x)$ be the set of all attackers of $x$ in \tuple{S',R'}.

We have if $\lambda'(x) \in \{\mbox{\inc, \exc}\}$, then $u \not\in
\mbox{ Att}'(x)$.

If $\lambda'(x) =\und$, then $y \in \mbox{ Att}'(u)$.

Consider the following set of equations on \tuple{S',R'}.

\[
x =1, \mbox{ if } \lambda'(x) =\mbox{\inc} \eqno\eqr{EQ1}
\]

\[
x=0, \mbox{ if } \lambda'(x)=\mbox{\exc} \eqno\eqr{EQ0}
\]

\[
x =\Pi(1-y)_{y\in Att'(x) \mbox{in } \tuple{S', R'}}, \mbox{ if }
\lambda'(x) =\und \eqno\eqr{EQU}
\]
This set of equations has a solution \Bf.

We claim the following
\begin{enumerate}
\item $\lambda(f)$ is a complete extension
\item $\lambda(f) =\lambda'$
\end{enumerate}

It is clear that $\lambda(f)(x)=\lambda'(x)$, for $\lambda'(x)\in
\{\mbox{\inc, \exc}\}$.  Does $\lambda(f)$ agree with $\lambda'$ on
undecided points of $\lambda'$?  The answer is that it must be so,
because $\lambda'$ is a preferred extension.  So $\lambda(f)$ cannot
be an extension with more zeros and ones than $\lambda'$.
\end{construction}

\begin{figure}[htb]
\centering 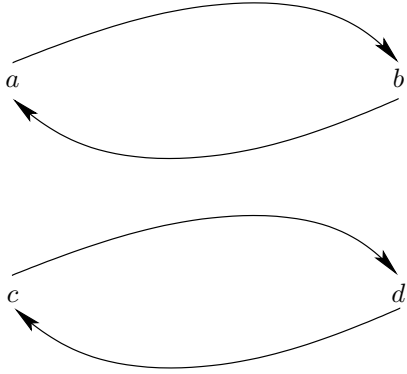
\caption{A network with two cycles.\label{532-F25}}
\end{figure}

\begin{remark}\label{532-R24}
The perceptive reader might ask why do we use those particular
equations in Construction~\ref{532-C23} (page~\pageref{532-C23})?  
The answer can be seen from Figure \ref{532-F25}.

Consider $\lambda(a)=\inc$, $\lambda(b)=\exc$,
$\lambda(c)=\lambda(d)=\und$.

We create Figure \ref{532-F26}.

\begin{figure}
\centering 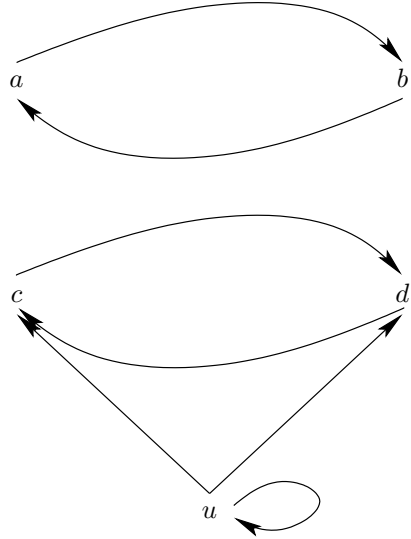
\caption{A self-attacking node attacking one of the cycles
in the network of Figure~\ref{532-F25}.}\label{532-F26}
\end{figure}

We take the equation
\[\begin{array}{l}
a=1, b=0\\[1ex]
 c=(1-d)(1-u)\\[1ex]
 d=(1-c)(1-u)\\[1ex]
 u=1-u.
\end{array}
\]

The solution for the equations for $c,d$ and $u$ are
\[\begin{array}{l}
u=\half\\[1ex]
 c=d=\frac{1}{3}
\end{array}
\]

We have to insist on $a=1, b=0$. If we do not insist and write the
usual equations
\[\begin{array}{l}
a=1-b\\[1ex]
 b=1-a,
\end{array}
\]
we might get a different solution, e.g.
\[
b=1, a=0.
\]
This not the original $\lambda$.
\end{remark}

\begin{remark}\label{532-R27}
This remark motivates and proves the next Theorem \ref{532-T28}.  We
need some notation. Let $Q$ be a set of atoms. By the models of $Q$ 
(based on $Q$) we mean all conjunction normal forms of atoms from $Q$ 
or their negations. So, for example, if $Q = \{ a, b, c \}$, we get 8 models,
namely
\[\begin{array}{c}
m_1=a\wedge b\wedge c\\[1ex]
 \vdots\\[1ex]
 m_8=\neg a\wedge\neg b\wedge\neg c.
\end{array}
\]
If we have atoms
\[
Q_1 =\{a_i\}, Q_2 =\{b_j\}, Q_3=\{c_k\}
\]
where $Q_i$ are pairwise disjoint we can write the models of $Q_1\cup
Q_2\cup Q_3$ in the form
\[
\alpha\wedge\beta\wedge \gamma
\]
where $\alpha$ is a model of $Q_1, \beta$ of $Q_2$ and $\gamma$ of
$Q_3$.

For example
\[
\alpha_1\wedge\beta_1\wedge\gamma_1=(a_1\wedge a_2\wedge\ldots)
\wedge(\neg b_1\wedge b_2\wedge\ldots) \wedge (c_2\wedge\ldots).
\]

Now let \tuple{S,R} and $\lambda$ be as in Construction \ref{532-C23}.
Remember we assume that the value of $u$ is very very small, and so
the attack value $(1-u)$ is very close to $1$. Consider $\lambda'$ and $f$ and
$\lambda(f)$ again as in Construction \ref{532-C23}.  $f$ is a
solution of \eqinv\ equations~\eqr{EQ1}, \eqr{EQ0} and \eqr{EQU}.  Therefore any
model of $S'$, say $\alpha =\pm s_1\wedge \pm s_2\wedge\pm \ldots
\wedge\pm s_k\wedge \pm u$ where $S=\{s_1\comma s_k\}$ will have its
probability semantics as
\[
P_f(\alpha =\Pi^k_{i=1} f(\pm s_k))) \times f(\pm u) \eqno (*)
\]
where
\[\begin{array}{l}
f(+s) =f(s)\\[1ex]
 f(-s) =1-f(s).
\end{array}
\]

In particular, we have the following:
\begin{enumerate}
\item Let $E^+ =\{e^+_1,\ldots\}$ be the subset of $S$ such that
  $\lambda (e^+_i)=\inc$. Let $E^- =\{e^-_j\}$ be the subset of $S$
  such that $\lambda (e^-_j)=\exc$.  Let $E_{\rm und} = \{b_k\}$ be
  the set of all nodes in $S$ such that $\lambda (b_k)=\und$.

We therefore have that any model $\delta$ of $S'$ has the form
\[\begin{array}{lcl}
\delta &=& \bigwedge_i \pm e^+_i\wedge\bigwedge_i\pm
e^-_j\wedge\bigwedge_k \pm b_k\wedge \pm u\\[1ex]
 &=& \alpha \wedge\beta
\pm u
\end{array}
\]
where $\alpha$ is a model of $E^+\cup E^-$ and $\beta$ is a model of
$E_{\rm und}$.

Let $\alpha_{1,0}$ be the particular conjunction
\[
\alpha_{1,0} =\bigwedge_i e^+_i\wedge\bigwedge_j \neg e^-_j.
\]
Let $\beta$ be any model of $E_{\rm und}$.  Consider $P_f(\delta),
\delta = \alpha \wedge \beta\wedge\pm u$.  Then by (*) we have that
\[
P_f(\delta) =0, \mbox{ if } \alpha\neq \alpha_{1,0}.  \eqno (**)
\]

Since $P_f$ is a probability, we have for any $s\in S'$
\[
P_f(s) = P_f (\bigwedge_{y\in Att'(s)} \neg y).
\]
Note that for $s\in S, s\neq u$ such that $\lambda (s) \in \{\mbox{\inc,
  \exc}\}$, $u$ does not attack $s$, and so we have
\[\begin{array}{lcl}
P_f (s)& = &P(f) (\bigwedge_{y\in Att(s)}\neg y)\\[1ex]
 &=& \Pi_{y\in
  Att(s)} (1-f(y))
\end{array}
\eqno (\sharp 1)
\]

For $u$ we have that $u$ is very small and so $P_f(u)
=\frac{1}{2^{n}}$.

For $s\in S$ such that $\lambda (s)=\und$, we have that $u$ attacks
$s$ and so
\[\begin{array}{lcl}
P_f(s) &=& P_f(\bigwedge_{y\in Att'(s)} \neg y)\\[1ex]
 &=& (\Pi_{y\in
  Att(s)} (1-f(y)) \times (1-\frac{1}{2^{n}})
\end{array}
\eqno (\sharp 2)
\]
The $(1-\frac{1}{2^{n}})$ is the attack of $u$.
\end{enumerate}
We ask what are the attackers of $s \in E_{\rm und}$?  They cannot be
nodes $x$ such that $\lambda (x)=$ in, because then $s$ would be out.
So the value of $f(y)$, (for $y \in Att(s)$) is either $0$ or a value in
$(0,1)$.

So we can continue and write
\[
P_f(s) =(1-\frac{1}{2^n}) \Pi_{\parbox{1.5cm}{\scriptsize$y\in Att(s)$\\
$\lambda(y)=\und$}} (1-f(y)) \eqno (\sharp 3)
\]
Note that $0 < P_f(s) < 1$, because all the $f(y)$, for $\lambda
(y)=\und$, satisfy $0 < f(y) < 1$.

We also have
\[
\sum_{\mbox{all models } m} P_f(m) =1.  \eqno (\sharp 4)
\]
Since(**) holds, we need consider only models $m$ of the form
$\alpha_{1,0} \wedge\beta\wedge\pm u$.

We can write
\[
1=\sum_{\beta \wedge\pm u} P_f (\alpha_{1,0} \wedge\beta \wedge\pm u)
\eqno (\sharp 5)
\]
where $\beta$ is a model of $E_{\rm und}$.  Let us analyse $(\sharp
5)$ a bit more.

Assume $\beta =\bigwedge_k \pm b_k$.

So
\[
P_f (\alpha_{0,1} \wedge\beta\wedge u )+P_f (\alpha_{0,1}
\wedge\beta\wedge\neg u) =\Pi_k f(\pm b_k).  \eqno (\sharp 6)\] We
thus get that:
\[
\sum_\beta \Pi_k f(\pm b_k) =1.  \eqno (\sharp 7)
\]

$(\sharp 7)$ says something very interesting. It says that $f$
restricted to $E_{\rm und}$ gives a proper probability distribution on
the models of $E_{\rm und}$.

This combined with $(\sharp 3)$ gives us the following result.

Consider $(E_{\rm und}, R_{\rm und})$ where $R_{\rm und} =
R\upharpoonright E_{\rm und}$. Then $f\upharpoonright E_{\rm und}$ is
a proper probability distribution on $(E_{\rm und}, R_{\rm und})$.

Does it satisfy the proper equations?

Let $s\in E_{\rm und}$.  Do we have
\[
P_{\rm und} (s) \stackrel{?}{=} P_{\rm und}(\bigwedge_{\begin{array}{c}y\in E_{\rm
      und}\\
yRx
\end{array}} \neg y)
\]
Let us check.

The real equation is
\[
P_{\rm und} (s) = P_{\rm und}(\bigwedge_{
\begin{array}{c}
y'\in E_{\rm und}\\[1ex]
yRx\end{array}} \neg y) \times (1-u) \eqno (\sharp
8)\] Since $u$ is very small, we have a very good
approximation.\footnote{The perceptive reader might ask what happens
  if we let $u$ converge to 0?  The answer is that we get a proper
  $\eqinv$ extension. However, this may be an all undecided extension
  (which is what we do want), or it may be a complete extension
  properly containing all the undecided extensions (which is not what
  we want!).

We may decide to do what physicists do to their equations. Write the
equations in full and simply neglect any item containing higher order
$u$, i.e., $u^2, u^3$, etc. This is reasonable when the value of each
node is small.}

We can now define a probability $P$ on \tuple{S,R}.  Let $m
=\alpha\wedge\beta$ be a model, where $\alpha$ is a model for $E^+\cup
E^-$ and $\beta$ is a model for $E_{\rm und}$.

Then define $P$ as follows
\[\begin{array}{l}
P(\alpha\wedge\beta) = 0,\mbox{ if } \alpha=\neg
\alpha_{1,0}\\[1ex]
 P(\alpha\wedge\beta) = P_{\rm und} (\beta),\mbox{ if }
\alpha = \alpha_{1,0}
\end{array}
\]
We need to show that approximately
\[
P(s) = P(\bigwedge_{y\in Att(s)} \neg y)
\]
If $s\in E^+\cup E^-$ this follows from $(\sharp 1)$.

If $s\in E_{\rm und}$, this follows from $(\sharp 3)$ and $(\sharp
8)$.

Note that since the $f$ involved came from \eqinv\ equations, $P$
satisfies the following on \tuple{S,R}.
\[\begin{array}{l}
P(s) =0, \mbox{ if some } y \in Att(s) P(y)=1\\[1ex]
 P(s) =1, \mbox{ if for
  all } y \in Att(s), P(y) =0\\[1ex]
 P(s) =\mbox{ undecided, otherwise}.
\end{array}
\eqno (\sharp 9)
\]
\end{remark}

\begin{theorem}\label{532-T28}{\ }
\begin{enumerate}
\item Let \tuple{S,R} be a network and let $\lambda$ be a legitimate
  Caminada labelling on $S$. Then there exists a Method 1 probability
  distribution $P_\lambda$, which almost satisfies equation \eqr{E3}, 
  namely for every $\varepsilon$, there exists a Method 1 probability $P_\lambda$ 
  depending on $\varepsilon$, such that for every $x$ and its attackers $y_i$, 
  we have $|P_\lambda(x) - P_\lambda  (\wedge \neg y_i)| <\varepsilon$, such
  that
\[\begin{array}{l}
\lambda (x)=\inc,\mbox{ if } P_\lambda (x) =1\\[1ex]
 \lambda (x)=
\exc,\mbox{ if } P_\lambda (x) =0\\[1ex]
 \lambda (x)=\und,\mbox{ if } 0 <
P_\lambda (x) < 1.
\end{array}
\]

\item $P$ is obtained as follows

\paragraph{Case 1.} $\lambda$ is a preferred extension.  Then let $f$ be a solution of \eqinv\  for \tuple{S,R}. Let $P_\lambda = f$.
\paragraph{Case 2.} $\lambda$ is not a preferred extension.  

Let $E^\lambda_{\rm und}=\{x|\lambda (x) = {\rm und}\}$.  Consider
\tuple{S',R'}, where $S' =E^\lambda_{\rm und}\cup \{u\}$, where $u$ is a
new point not in $S$ with value almost 0.
\[
R' = R\upharpoonright E^\lambda_{\rm und} \cup \{u\} \times
E^\lambda_{\rm und}.
\]

Then \tuple{S',R'} has only one extension (all undecided). Let $f'$ be a
solution to \eqinv\ on \tuple{S',R'}.  We now define $P_\lambda$ on
\tuple{S,R}.

Let $\alpha_{1,0} =\bigwedge_{\lambda (x) =\mbox{ in}} x
\wedge\bigwedge_{\lambda(y)=\mbox{ out}} \neg y$.

Let $m = \alpha \wedge\beta$ be an arbitrary model of $S$, where
$\alpha$ is a model of $\{x|\lambda (x) \in \{\mbox{\inc, \exc}\}$ and
$\beta$ is a model of $E^\lambda_{\rm und}$.  Define $P_\lambda
(\alpha\wedge\beta)$ to be
\[\begin{array}{l}
P_\lambda (\alpha\wedge\beta =0\mbox{ if } \alpha\neq
\alpha_{1,0}\\[1ex]
 P_\lambda (\alpha_{1,0} \wedge\beta) =
f'(\beta)\\[1ex]
 \mbox{where } \beta =\bigwedge_{s\in E^\lambda_{\rm und}}
\pm s\\[1ex]
 \mbox{and } f(\beta) =\Pi_{\pm s \mbox{ in } \beta} f(\pm s).
\end{array}\]
\end{enumerate}
\end{theorem}

\begin{proof}
Follows from the considerations of Remark \ref{532-R27}.
\end{proof}

\begin{example}\label{532-E29}
Let us show how Theorem \ref{532-T28} works by doing a few examples.
\begin{enumerate}
\item Consider the network of Figure \ref{532-F25} and the extension
  $\lambda$ mentioned there, namely $\lambda(a) =\inc$, $\lambda (b)
  =\exc$, $\lambda ( c) =\lambda (d) =\und$.

Following our algorithms we look at the $\{c, d, u\}$ part of Figure
\ref{532-F26} and solve the equations. We get $u = \half,
c=d=\frac{1}{3}$.

The probability $P_\lambda$ will be as follows:
\[\begin{array}{l}
P_\lambda (\alpha\wedge\beta)=0\\[1ex]
 \mbox{ if } \alpha\neq a\wedge\neg
b.
\end{array}
\]

Now look at
\[\begin{array}{l}
\renewcommand{\arraystretch}{1.5} P_\lambda (a\wedge\neg b\wedge
c\wedge d)=\frac{1}{3}\times \frac{1}{3}=\frac{1}{9}\\[1ex]
 P_\lambda
(a\wedge\neg b\wedge c\wedge\neg d) =\frac{1}{3}\times \frac{2}{3}
=\frac{2}{9}\\[1ex]
 P_\lambda (a\wedge \neg b \wedge\neg c \wedge d)
=\frac{2}{3}\times \frac{1}{3}=\frac{2}{9}\\[1ex]
 P_\lambda (a\wedge\neg
b\wedge\neg c\wedge\neg d) =\frac{2}{3}\times \frac{2}{3}=\frac{4}{9}.
\end{array}
\]
\item Let us look at Figure \ref{532-F30}.

\begin{figure}
\centering 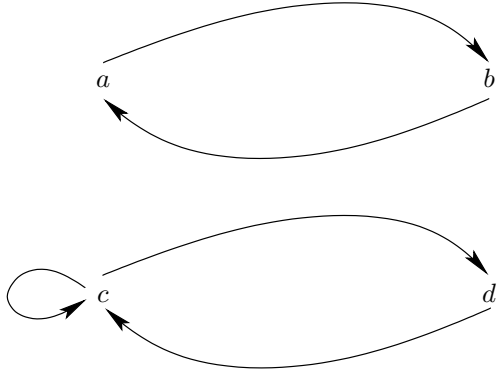
\caption{Augmented network of Figure~\ref{532-F3} with node
$a$ as $c$ and $b$ as $d$ and an extra cycle.\label{532-F30}}
\end{figure}

With $\lambda (a)=\inc$, $\lambda(b)=\exc$, $\lambda ( c)=\lambda
(d)=\und$.

The $\{c,d\}$ part is Figure \ref{532-F3}.  Here we solve the
equations on the $\{c,d,u\}$ part associated with $\{c,d\}$, which is
the same as Figure \ref{532-F4}.  The solution is found in Example
\ref{532-f4}, with $u=\half$.

We get $u=\half; c = 0.36, 1-c = 0.764, d = 0.382, 1-d = 0.618$.  The
probability $P_\lambda$ of this case is $P_\lambda
(\alpha\wedge\beta)=0$, if $\alpha\neq a\wedge\neg b$.

\[\begin{array}{l}
P_\lambda (a\wedge\neg b\wedge c\wedge d)=0.236 \times
0.382=0.09\\[1ex]
 P_\lambda (a\wedge\neg b\wedge c\wedge\neg d) = 0.236
\times 0618 = 0.146\\[1ex]
 P_\lambda (a\wedge\neg b\wedge\neg c \wedge d) =
0.764 \times 0.382=0.292\\[1ex]
 P_\lambda (a\wedge\neg b\wedge\neg
c\wedge\neg d) = 0.764 \times 0.618 = 0.472.
\end{array}
\]

Indeed
\[
0.09 + 0.146 + 0.292 + 0.472 = 1.000.
\]

\end{enumerate}
\end{example}

We now discuss imposing probability on instantiated networks such as ASPIC+. We
begin with simple instantiations into classical propositional logic. 

\begin{definition}\label{def:532-D51}
\begin{enumerate}
\item \label{it:def-AIN-1}
An abstract instantiated network (into classical propositional logic) has
the form ${\cal A}=\tuple{S,R,I}$, where \tuple{S,R} is an abstract argumentation
network and $I$ is a mapping associating with each $x \in S$, a well-formed
formula $I(x)=\varphi_x$ of classical propositional logic.

\item For any $\cal A$ as in \ref{it:def-AIN-1}, we associate the theory
$\Delta_{\cal A}=\{ \varphi_x \iff \wedge_{(y,x) \in R}\neg \varphi_{y} \; 
|\; x \in S\}$.

\item A semantic probability model $P$ on $\cal A$ is a probability distribution
on the models based on $S$ such that for all $x \in S$, we have:
\[P(\varphi_x)=P(\wedge_{(y,x)\in R} \neg \varphi_y)\]
\end{enumerate}
\end{definition}

\begin{example}\label{ex:532-E52} 
Consider Figure~\ref{fig:532-F53} where
part (b) is an instatiation of part (a) with $I(x)=a_1 \vee a_2$ and $I(a_3)=
a_3$. The equations any probability assignment needs to satisfy are
\begin{eqnarray*}
P(a_1 \vee a_2) & = & 1\\[1ex]
P(a_3) & = & P(\neg (a_1 \vee a_2))\\[1ex]
& = & P(\neg a_1 \wedge \neg a_2)\\[1ex]
& = & 0.
\end{eqnarray*}
If we let $P(a_1)=x$, $P(a_2)=1-x$, $P(a_3)=0$, with $x\in [0,1]$, then $P$ 
satisfies the equations. Compare with Example~\ref{532-E8}.

\begin{figure}[htb]
\begin{center}
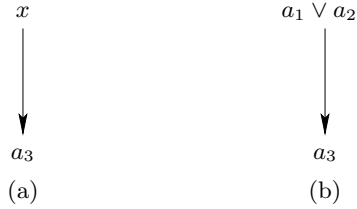
\end{center}
\caption{(a) A network and (b) one of its instantiations with 
$x=a_1\vee a_2$\label{fig:532-F53}}
\end{figure}
\end{example}

\section{Comparison with the literature}
There are several probabilistic argumentation papers around.  This is
a hot topic in 2014.  We highlight two main points of view. The
external and the internal views.

Let \tuple{S,R} be a network and let \Bf\ be a function from $S$ to
$[0,1]$.  We can regard \Bf\ as giving a probability number to each
$x\in S$.   
The internal probability is where the above numbers signify the value
of the argument.  Its truth, its reliability, its probability of being
effective, etc., or whatever measure we attach to it as an
argument. Figure \ref{532-F50} represents in this case the $\eqinv$
solution (and hence probability) of the network of Figures
\ref{532-F3} and \ref{532-F4}.
The external view is to think of $\Bf(x)$ as the probability
of the predicate ``$x\in S$''.  That is, the probability that the
argument $x$ is present in $S$.  Consider again Figure \ref{532-F50}.

\begin{figure}
\centering 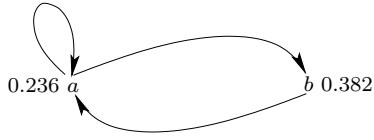
\caption{The \eqinv\ solution to the networks of Figures~\ref{532-F3} 
and~\ref{532-F4}}\label{532-F50}
\end{figure}

The probability that $a$ is in the network is 0.236 and the
probability that $b$ is in the network is 0.382.  Therefore, the
probability that the network contains both $\{a,b\}$ is 0.236 $\times$
0.388 = 0.09.    The probability that the network contains only $a$ is  $0.236 \times (1- 0.382) = 0.1458$. The probability that the network contains only $b$ is $0.382 \times  (1 - 0.236) = 0.292$ and the probability that the network is empty is $(1 - 0.236) \times  (1 - 0.382) = 0.472$.
  It is clear why we are calling this view an external
probability view.   It imposes probability externally expressing uncertainty on what the network graph is. This is done either by giving the probability to points or more generally by giving probability directly to subsets $G$ of $S$, expressing the probability that the graph is really that subset of $S$ with $R$ restricted to $G$.  This external view has value in dialogue argumentation or negotiation when we try to estimate what network our opponent is reasoning with.  The problem with this external view is how to connect with the attack relation. Note that mathematically in the external view we have probabilities on points in $S$ or probabilities on subsets of $S$, which are the same options as in our internal view, but the understanding of  them is different. We in the internal view considered the subset as a classical model, while the external view considers it as a subnetwork. 
When we use the internal view, we can connect it with the attack relation via the equational approach (Equation~\eqr{E3}), but how would the external view connect with the attack relation?  We can ask, for example, how to get a value for a single point to be ``in'' an extension?
Intuitively, looking back at Figure \ref{532-F50},  we can say the point a for example is ``in''  in case the network is $\{a\}$ and is also ``in"  in one of the three extensions in case the network is $\{a,b\}$. So we might take the ``in" value to be  $0.1458 + 0.09/3 = 0.1458 + 0.03 = 0.1758$.   The connection with the attack relation can be done perhaps through the probabilities for admissible sets, since being admissible is connected with the attack relation. There are problems, however, with this approach.

Hunter \cite{hunter:12} was trying to lay some foundations for this view, following the papers \cite{14,15}.  See also a good summary in  Hunter\cite{hunter:13}.  Hunter was trying to find a connection between the external probability view and some reasonable values we can give to admissible subsets. He proposes restrictions on the probability function on $S$.  We are not going to discuss or reproduce Hunter's arguments here. It suffices to say that  possibly a subsequent paper of ours will critically examine the external view and compare with the internal view.

Let us now  compare our work with that of M. Thimm, \cite{thimm:12}, whose
approach is also internal.  We quote from \cite{thimm:12}:
\begin{quotation}
``In this paper we use another interpretation for probability, that of
{\em subjective probability} \cite{15a}.  There, a probability $P(X)$
for some $X\in \CX$ denotes the {\em degree of belief} we put into
$X$. Then a probability function $P$ can be seen as an epistemic state
of some agent that has uncertain beliefs with respect to $\CX$.  In
probabilistic reasoning \cite{15a,16a}, this interpretation of
probability is widely used to model uncertain knowledge representation
and reasoning.

In the following, we consider probability functions on sets of
arguments of an abstract argumentation frameworks.  Let ${\sf AF} =
({\sf Arg}, \break \to)$ be some fixed abstract argumentation framework and
let $\CE =2^{\sf Arg}$ be the set of all sets of arguments. Let now
$\CP_{\sf AF}$ be the set of probability functions of the form $P:
2^\vare \to [0,1]$.  A probability function $P\in \CP_{\sf AF}$
assigns to each set of possible extensions of {\sf AF} a probability,
i.e.\ $P(e)$ for $e\in \CE$ is the probability that $e$ is an
extension and $P(E)$ for $E\subseteq\CE$ is the probability that any
of the sets in $E$ is an extension.  In particular, note the difference
between e.g.\ $P(\{\CA, \CB\})=P(\{\{\CA,\CB\}\})$ and
$P(\{\{\CA\},\{\CB\}\})$ for arguments $\CA,\CB$.  While the former
denotes the probability that $\{\CA,\CB\}$ is an extension the latter
denotes the probability that $\{\CA\}$ or $\{\CB\}$ is an
extension. In general, it holds $P(\{\CA,\CB\})\neq
P(\{\{\CA\},\{\CB\}\})$ .

For $P\in \CP_{\sf AF}$ and $\CA\in {\sf Arg}$ we abbreviate
\[
P(\CA) =\sum_{\CA\in e\subseteq {\sf Arg}}P(e).
\]
Given some probability function $P$, the probability $P(\CA)$
represents the degree of belief that $\CA$ is in an extension
(according to $P$), i.e.\, $P(\CA)$ is the sum of the probabilities of
all possible extensions that contain $\CA$.  The set $\CP_{\sf AF}$
contains all possible views one can take on the arguments of an
abstract argumentation framework {\sf AF}.

\noindent {\bf Example 4.}  We continue Ex. 1. (Comment by Gabbay and
Rodrigues: This is the network of our Figure 4.)  Consider the
function $P\in \CP_{\sf AF}$ defined via $P(\{\CA_1,\CA_3,\CA_5\}) =
03, P(\{\CA_1,\CA_4\})=0.45, P(\{\CA_5,\CA_2\})=0.1,
P(\{\CA_2,\CA_4\})=0.15$, and $P(3)=0$ for all remaining $e\in \CE$.
Due to Prop. 1 the function $P$ is well-defined as in, e.g.,
\[\begin{array}{l}
P(\{\{\CA_5,\CA_2\}, \{\CA_2,\CA_4\}, \{\CA_3\}\})\\[1ex]
 \quad =
P(\{\CA_5,\CA_2\}) +P(\{\CA_2,\CA_4\}) +P(\{\CA_3\})\\[1ex]
 \quad = 0.1 +
0.15 + 0=0.25.
\end{array}
\]
Therefore, $P$ is a probability function according to Def. 3.
According to $P$ the probabilities to reach argument of {\sf AF}
compute to $P(\CA_1)= 0.75, P(\CA_2)= 0.25, P(\CA_3 =0.3, P(\CA_4)
=0.6$, and $P(\CA_5) = 0.4$.

In the following, we are only interested in those probability
functions of $\CP_{\sf AF}$ that agree with our intuition on the
interrelationships of arguments and attack.  For example, if an
argument $\CA$ is not attacked we should completely believe in its
validity if no further information is available.  We propose the
following notion of {\em justifiability} to describe this intuition.

\noindent {\bf Definition 4.}  A probability function $P \in \CP_{\sf
  AF}$ is called $p$-{\em justifiable} wrt. {\sf AF}, denoted by
$P\Vdash _{\CJ} {\sf AF}$, if it satisfies for all $\CA\in {\sf Arg}$.
\begin{enumerate}
\item $P(\CA)\leq 1 - P(\CB)$ for all $\CB, \in {\sf Arg}$ with $\CB\to
  \CA$ and
\item $P(\CA)\geq 1 - \sum_{\CB\in\CF} P(\CB)$ where
  $\CF=\{\CB|\CB\to\CA\}$.
\end{enumerate}
Let $P^{\CJ}_{\sf AF}$ be the set of all $p$-justifiable probability
functions wrt. {\sf AF}.

The notion of $p$-justifiability generalizes the concept of complete
semantics to the probabilistic setting. Property 1.) says that the
degree of belief we assign to an argument $\CA$ is bounded from above
by the complement to $1$ of the degrees of belief we put into the
attackers of $\CA$. As a special case, note that if we completely believe
in an attacker of $\CA$, i.e., $P(\CB) = 1$ for some $\CB$ with 
$\CB\to \CA$, then it follows $P(\CA) = 0$. This corresponds to 
property 1.) of a complete labelling (see Section 2). Property 2.) 
of Def. 4 says that the degree of belief we assign to an 
argument $\CA$ is bounded from below by the inverse of 
the sum of the degrees of belief we put into the attacks of $\CA$.
As a special case, note that if we completely disbelieve in 
all attackers of $\CA$, i.e. $P(\CB)=0$ for all $\CB$ with $\CB \to \CA$, 
then it follows $P(\CA) = 1$. This corresponds to property 2.) of a complete
labeling, see Section 2. The following proposition establishes the
probabilistic analogue of the third property of a complete labelling.

\noindent {\bf Proposition 2.}  Let $P$ be $p$-justifiable and $\CA\in
          {\sf Arg}$.  If $P(\CA) \in (0,1)$ then
\begin{enumerate}
\item there is no $\CB\in {\sf Arg}$ with $\CB\to \CA$ and $P(\CB)=1$
  and 
\item there is a $\CB'\in {\sf Arg}$ with $\CB'\to \CA$ and
  $P(\CB')>0$.
\end{enumerate}
\end{quotation}

 From our point of view, Thimm's approach is a variant of our semantic
 Method 2 approach without the strong equation \eqr{E3} but the weaker
 Definition 4 of Thimm. Thus Thimm will allow for different values for
 nodes $x_1$ and $x_2$ in our Figure \ref{532-F22}, while we would not
 (see Example \ref{532-E10}).
 
 Although Thimm's approach is mathematically close to us, conceptually
 we are far apart.  Thimm motivates his approach as a degree of belief
 in a subset $E\subseteq S$, considering $E$ as an extension. We
 consider $E$ as representing a classical model $m$ of the classical
 propositional logic with atoms $S$
 \[
 m = \bigwedge_{s\in E} s\wedge\bigwedge_{s\not\in E}\neg s
 \]
 and assign probability to it and then we export this probability to
 argumentation via the equational approach, equation \eqr{E3}.
 
 This is an instance of our methodology of ``Logic by Translation'',
 From our point of view, equations \eqr{E3} are essential, conceptual and
 non-technical. For Thimm, the inequalities of his Definition 4 appear
 to be technical to enable the probabilities to work of ground
 extension.
 
 Our point of view also leads us to the $\eqinv$ Method 1
 probabilities and to the approximation results of Section 4.
 
 In Thimm's conceptual approach, this way of thinking does not even
 arise.
 
To summarise, this paper presented an internal view of probabilistic argumentation.
There is a need for two subsequent research papers
 \begin{enumerate}
 \item The external view done coherently and its connection to the internal view
 \item  A conditional probability view and its connection with Bayesian Networks views as Argumentation Networks
 \end{enumerate}
  
\section{Conclusions}
\newcommand{\BBS}{\mbox{\(\Bbb S\)}}
\newcommand{\BBA}{\mbox{\(\mathbb{A}\)}}

This section explains and sets our approach in a general generic context.

Suppose we are given a system $\BBS$ such as an argumentation system \tuple{S,R} and we want to add to it some aspect $\BBA$.

There is a generic way to add any new feature to a system. It involves 1) identifying the basic units which build up the system and 2) introducing the new feature to each of these basic units. In the case where the system is argumentation and the feature is probabilistic we have the following: the basic units are {\bf a.} the nature of the arguments involved; {\bf b.} the membership relation in the set $S$ of arguments;\footnote{Note that the set $S$ itself may not be fully or accurately known, especially modelling an opponent in dialogue systems.} {\bf c.} the attack relation; and {\bf d.} the choice of extensions.

Generically to add a new aspect  (probabilistic, or fuzzy, or temporal, etc) to an argumentation network \tuple{S,R} can be done by adding this feature to each component. {\bf a.} We make the effective strength of the argument probabilistic; {\bf b.} we give probability to whether an argument is included in $S$;\footnote{{\bf a.} and {\bf b.} are distinct, because {\bf a.} represents how effective an argument is, whereas {\bf b.} is the decision of whether or not to include an argument for consideration. An argument may be deemed very effective but not included for consideration for completely different reasons.} {\bf c.} we make the attack relation probabilistic; and {\bf d.} we put probability on the extensions.

These features interact and need to be chosen with care and coordination. We need a methodological approach to make our choices. One such methodology is what we called “logic by translation”.

We meaningfully translate the argumentation system into classical logic which does have probabilistic models and then  let probabilistic classical logic endow the probability on the argumentation system. As we mentioned, this of course depends on how we translate.

We gave in this paper an object-level translation. The arguments of $S$ became atoms of classical propositional logic, we then used probability on the models of classical logic and used the attack relation $R$ to express equational restrictions on the probabilities. In this kind of translation, the attack relation did not become probabilistic.

We could have used a meta-level translation into classical predicate logic, using a binary relation $R$ for expressing in classical logic the attack relation and using unary predicates to express that an argument $x$ is ``in'', $x$ is ``out'', etc., with suitable coordinating axioms. In this case all predicates would have become probabilistic including the attack relation $R$. As far as we know nobody has done this to $R$.

In this context of possible  options what we have done is one systematic approach and we compared it with other approaches. It should be noted that we could have followed the same steps to get fuzzy argumentation networks; temporal argumentation networks; or indeed any other feature available for classical propositional logic.

\bibliographystyle{plain} 
\bibliography{532-bib}

\end{document}